%% file: sum2017.tex
\documentclass[final,envcountsame]{llncs}

\usepackage{times}
\usepackage{helvet}
\usepackage{courier}
\usepackage{amsmath,amssymb}
\usepackage[all]{xy}
\usepackage{url}
\usepackage[obeyDraft]{todonotes}
\usepackage{stmaryrd}
\usepackage{mathtools}
\usepackage{xspace}
\usepackage{multirow}
\usepackage{nicefrac}
\usepackage{booktabs}
\usepackage{cite}
\usepackage{tikz}
\usetikzlibrary{arrows,positioning,shapes.misc,shapes.geometric}
\allowdisplaybreaks  

\def\ProofVersion{1}

\newcommand{\conc}[1]{\scalebox{0.9}{\ensuremath{\mathsf{#1}}\xspace}}

\newcommand{\ALC}{\ensuremath{\mathcal{ALC}}\xspace}
\newcommand{\el}{\ensuremath{\mathcal{E\!L}}\xspace}
\newcommand{\interp}{\ensuremath{\mathcal{I}}\xspace}
\newcommand{\kb}{\ensuremath{\mathcal{K}}\xspace}
\newcommand{\modelSet}{\ensuremath{\mathcal{\operatorname{Mod}}}\xspace}
\newcommand{\sig}{\ensuremath{\mathcal{\operatorname{Sig}}}\xspace}
\newcommand{\countDom}[1]{\ensuremath{[#1]^\interp}}
\newcommand{\countDomP}[2]{\ensuremath{[#1]^{#2}}}

\newcommand{\NC}{\ensuremath{{\sf N_C}}\xspace}
\newcommand{\NR}{\ensuremath{{\sf N_R}}\xspace}

\newcommand{\Cmc}{\ensuremath{\mathcal{C}}\xspace}
\newcommand{\Emc}{\ensuremath{\mathcal{E}}\xspace}
\newcommand{\Fmc}{\ensuremath{\mathcal{F}}\xspace}
\newcommand{\Imc}{\ensuremath{\mathcal{I}}\xspace}
\newcommand{\Kmc}{\ensuremath{\mathcal{K}}\xspace}
\newcommand{\Lmc}{\ensuremath{\mathcal{L}}\xspace}
\newcommand{\Smc}{\ensuremath{\mathcal{S}}\xspace}
\newcommand{\Tmc}{\ensuremath{\mathcal{T}}\xspace}
\newcommand{\Vmc}{\ensuremath{\mathcal{V}}\xspace}
\newcommand{\rL}[1]{\ensuremath{\bf L_{#1}}\xspace}
\newcommand{\rC}[1]{\ensuremath{\bf C_{#1}}\xspace}

\begin{document}

\ifnum\ProofVersion=0
\title{Towards Statistical Reasoning in Description Logics over Finite Domains}
\fi
\ifnum\ProofVersion=1
\title{Towards Statistical Reasoning in Description Logics over Finite Domains (Full Version)}
\fi
\author{
Rafael Pe\~naloza\inst{1} 
\and  
Nico Potyka\inst{2}
}

\institute{KRDB Research Centre, Free University of Bozen-Bolzano, Italy \\
			\email{rafael.penaloza@unibz.it}
            \and 
            University of Osnabr\"uck, Germany \\						
            \email{npotyka@uni-osnabrueck.de} }

\maketitle

\begin{abstract}
We present a probabilistic extension of the description logic \ALC for reasoning about
statistical knowledge. We consider conditional statements over proportions of the domain
and are interested in the probabilistic-logical consequences of these proportions. 
After introducing some general reasoning problems and analyzing their properties, 
we present first algorithms and complexity results for reasoning in some fragments of Statistical \ALC. 
\end{abstract}

\section{Introduction}

Probabilistic logics enrich classical logics with probabilities in order to incorporate uncertainty.
In \cite{halpern1990analysis}, probabilistic logics have been classified into three types that differ
in the way how they handle probabilities.
Type 1 logics enrich classical interpretations with probability distributions over the domain and 
are well suited for reasoning about statistical probabilities. This includes proportional statements
like ``$2 \%$ of the population suffer from a particular disease.'' 
Type 2 logics consider probability distributions over possible worlds and are better suited for expressing
subjective probabilities or degrees of belief. For instance, a medical doctor might say that she is $90 \%$
sure about her diagnosis.
Type 3 logics combine type 1 and type 2 logics allow to reason about both kinds of uncertainty.

One basic desiderata of probabilistic logics is that they generalize a classical logic.
That is, the probabilistic interpretation of formulas with probability 1 should agree with the classical interpretation.
However, given that first-order logic is undecidable, a probabilistic first-order logic that satisfies
our basic desiderata will necessarily be undecidable. In order to overcome the problem, 
we can, for instance, restrict to Herbrand interpretations over a fixed domain
\cite{Nilsson1986AI,Lukasiewicz2001,beierle2015extending}
or consider decidable fragments like description logics \cite{koller1997p,LuSt-JWS08,CePe17}.

Probabilistic type 2 extensions of description logics have been previously studied in~\cite{LutzS10}.
In the unpublished appendix of this work, a type 1 extension of \ALC is presented along with a
proof sketch for {\sc ExpTime}-completeness of the corresponding satisfiability problem.
This type 1 extension enriches
classical interpretations with probability distributions over the domain as suggested in \cite{halpern1990analysis}.
We consider a similar, but more restrictive setting here. We are interested in an \ALC extension
that allows statistical reasoning. However, we do not impose a probability distribution over the domain.
Instead, we are only interested in reasoning about the proportions of a population satisfying some given
properties.
For instance, given statistical information about the relative frequency of certain symptoms, diseases and 
the relative frequency of symptoms given diseases, one can ask the
relative frequency of a disease given a particular combination of symptoms.
Therefore, we consider only classical \ALC interpretations with finite domains and are interested in the relative
proportions that are true in these interpretations.

Hence, interpretations in our framework can be regarded as a subset of the
interpretations in \cite{LutzS10}, namely those with finite domains and a uniform probability 
distribution over the domain. These interpretations are indeed sufficient for our purpose.
In particular, by considering strictly less interpretations, we may be able to derive tighter answer
intervals for some queries.
Our approach bears some resemblance to the random world approach from~\cite{grove1992random}.
However, the authors in \cite{grove1992random} consider possible worlds with a fixed domain size $N$ and 
are interested in the limit of proportions as $N$ goes to infinity.  
We are interested in all finite possible worlds that satisfy certain proportions and ask what statistical
statements must be true in all these worlds.

We begin by introducing Statistical \ALC in Section \ref{sec_stat_alc} together with three
relevant reasoning problems. Namely, the Satisfiability Problem, the l-Entailment problem
and the p-Entailment problem. In Section \ref{sec_properties}, we will then discuss some logical properties
of Statistical \ALC.
In Section \ref{sec_stat_el} and \ref{sec_stat_pos}, we present first computational results for 
fragments of Statistical \ALC.
\ifnum\ProofVersion=0
We had to omit several proofs in order to meet space restrictions. All proofs can be found in the full version of the paper.\footnote{\url{http://www.inf.unibz.it/~penaloza/sum2017full.pdf}}
\fi

\input{sec_properties}

\section{Statistical \el}
\label{sec_stat_el}
Proposition \ref{prop:inclusion_and_det_conditionals} and the fact that reasoning in \ALC is 
{\sc ExpTime}-complete, show that our reasoning problems are {\sc ExpTime}-hard. However,
we did not find any upper bounds on the complexity of reasoning in \ALC so far.
We will therefore focus on some fragments of $\ALC$ now.

To begin with, we will focus on the sublogic \el~\cite{BaBL05} of \ALC that does not allow for negation
and universal quantification. 
Formally, \el concepts are constructed by the grammar rule
$C ::= A \mid \top \mid C\sqcap C \mid \exists r.C$,
where $A\in N_C$ and $r\in N_R$.	
A \emph{statistical \el KB} is a statistical \ALC KB where conditionals are restricted to \el concepts.
Notice that, due to the upper bounds in conditionals, statistical \el KBs are capable of 
expressing some weak variants of negations. For instance, a statement
$(C\mid \top)[\ell,u]$ with $u<1$ restricts every model $\Imc=(\Delta^\Imc,\cdot^\Imc)$ to contain
at least one element $\delta\in\Delta^\Imc\setminus C^\Imc$. Thus, contrary to classical \el, 
statistical \el KBs may be inconsistent.
\begin{example}
\label{exa:inconsistency}
Consider the KB $\Kmc_1=(\emptyset,\Cmc_1)$, where
\[
\Cmc_1 = \{ (A \mid \top)[0,0.2], \ (A \mid \top)[0.3,1]\}.
\]
Since $\top^\Imc=\Delta^\Imc\not=\emptyset$, 
every model $\Imc=(\Delta^\Imc,\cdot^\Imc)$ of $\Kmc_1$ must satisfy
\[
\countDom{A} \le 0.2 \countDom{\top} < 0.3 \countDom{\top} \le \countDom{A},
\]
which is clearly a contradiction. Thus, $\Kmc_1$ is inconsistent.
\end{example}
More interestingly, though, it is possible to simulate valuations over a finite set 
of propositional formulas wit the help of conditional statements. Thus, the satisfiability problem
is at least NP-hard even for Statistical \el. 
\begin{theorem}
\label{thm:nphard:el}
The satisfiability problem for Statistical \el is NP-hard.
\end{theorem}
\ifnum\ProofVersion=1
\begin{proof}
We provide a reduction from the well-known coNP-complete problem of deciding validity of a 
3DNF formula. Let $\varphi=\bigvee_{i=1}^n\kappa_i$ be a 3DNF formula; that is, each
$\kappa_i, 1\le i\le n$ is a conjunction of three literals 
$\kappa_i=\lambda_i^1\land\lambda_i^2\land\lambda_i^3$.
We construct a statistical \el KB as follows. Let $\Vmc$ be the set of all variables appearing in
$\varphi$. For every $x\in\Vmc$, we 
use two concept names $A_x$ and $A_{\neg x}$. In addition, for every clause $\kappa_i$ we 
introduce a concept name $B_i$, and create an additional concept name $C$.

Consider the KB $\Kmc_\varphi=(\Tmc_\varphi,\Cmc_\varphi)$, where
\begin{align*}
\Tmc_\varphi := {} & \{ \bigsqcap_{j=1}^3A_{\lambda_i^j} \sqsubseteq B_i, B_i\sqsubseteq C \mid 1\le i\le n \} \\
\Cmc_\varphi := {} & \{ (A_x \mid \top) [0.5,1], (A_{\neg x} \mid \top) [0.5,1], (A_{\neg x} \mid A_x) [0,0] \}
				\cup \{ (C \mid \top)[0,0.5] \}.
\end{align*}
Then it holds that $\varphi$ is valid iff $\Kmc_\varphi$ is inconsistent.
\qed
\end{proof}
\fi
%
\input{cons_nexp}

\section{Reasoning with Open Minded KBs}
\label{sec_stat_pos}

In order to regain tractability, we now further restrict statistical \el KBs by disallowing
upper bounds in the conditional statements. We call such knowledge bases open minded.
\begin{definition}[Open Minded KBs]
A statistical \el KB $\Kmc=(\Tmc,\Cmc)$ is \emph{open minded}
iff all the conditional statements $(C\mid D)[\ell,u]\in\Cmc$ are such that $u=1$.
\end{definition}
For the scope of this section, we consider only open minded KBs.
The first obvious consequence of restricting to this class of KBs is that negations cannot be 
simulated. In fact, every open minded KB is consistent and, as in classical \el, can be satisfied in a 
simple universal model.
\begin{theorem}
Every open minded KB is consistent.
\end{theorem}
\begin{proof}
Consider the interpretation $\Imc=(\{\delta\},\cdot^\Imc)$ where the interpretation function maps
every concept name $A$ to $A^\Imc:=\{\delta\}$ and every role name $r$ to
$r^\Imc:=\{(\delta,\delta)\}$. It is easy to see that this interpretation is such that
$C^\Imc=\{\delta\}$ holds for every \el concept $C$. Hence, \Imc satisfies all \el GCIs and 
in addition $[C\sqcap D]^\Imc=[C]^\Imc=1$ which implies that all conditionals are also 
satisfied.
\qed
\end{proof}
Recall that, intuitively, conditionals specify that a proportion of the population satisfies
some given properties. One interesting special case of $p$-entailment is the question how likely it is 
to observe an individual that belongs to a given concept.
\begin{definition}
Let \Kmc be an open minded KB, $C$ a concept, and $m\in[0,1]$. $C$ is \emph{$m$\mbox{-}ne\-ces\-sary}
in \Kmc if \Kmc p-entails $(C\mid\top)[m,1]$.
The problem of \emph{$m$-necessity} consists in deciding whether $C$ is
$m$-necessary in \Kmc.
\end{definition}
We show that this problem can be solved in polynomial time.
As in the previous section, we assume that the KB is in normal form and additionally, that 
all conditional statements $(A\mid B)[\ell,1]\in\Cmc$ are such that $\ell<1$. This latter assumption
is made w.l.o.g.\ since the conditional statement $(A\mid B)[1,1]$ can be equivalently replaced
by the GCI $B\sqsubseteq A$ (see Proposition~\ref{prop:inclusion_and_det_conditionals}).
Moreover, checking $m$-necessity of a complex concept $C$ w.r.t.\ the KB
$(\Tmc,\Cmc)$ is equivalent to deciding $m$-necessity of a new concept name $A$ w.r.t.\
the KB $(\Tmc\cup\{A\equiv C\},\Cmc)$.
Thus, in the following we consider w.l.o.g.\ only the problem of deciding $m$-necessity of a 
concept name w.r.t.\ to a KB in normal form.

Our algorithm extends the completion algorithm for classification of \el TBoxes to in addition keep
track of the lower bounds of necessity for all relevant concept names. 
The algorithm keeps as data structure a set \Smc of tuples of the form $(A,B)$ and $(A,r,B)$ for 
$A,B\in N_C\cup\{\top\}$. These intuitively express that the TBox \Tmc entails the subsumptions 
$A\sqsubseteq B$ and
$A\sqsubseteq \exists r.B$, respectively. Additionally, we keep a function \Lmc that maps every
element $A\in N_C\cup\{\top\}$ to a number $\Lmc(A)\in[0,1]$. Intuitively, $\Lmc(A)=n$ expresses that 
\Kmc p-entails $(A\mid \top)[n,1]$.

The algorithm initializes the structures \Smc and \Lmc as
\begin{align*}
\Smc := {} & \{ (A,A),(A,\top) \mid A\in N_C(\Kmc)\cup\{\top\}\} \\
\Lmc(A) := {} &
\begin{cases}
			0 & \text{if $A\in N_C(\Kmc)$} \\
			1 & \text{if $A=\top$}.
		\end{cases}
\end{align*}
These structures are then updated using the rules from Table~\ref{tab:rules}.
\begin{table}[tb]
\caption{Rules for deciding $m$-necessity}
\label{tab:rules}
\centering
\begin{tabular}{@{}l@{\quad}lll@{}}
\toprule
${\bf C_1}$ & {\bf if} $\{(X,A_1),(X,A_2)\}\subseteq \Smc$ 
	& {\bf and} $A_1\sqcap A_2\sqsubseteq B\in \Tmc$
	& {\bf then} add $(X,B)$ to \Smc \\
${\bf C_2}$ & {\bf if} $(X,A)\in \Smc$ 
	& {\bf and} $A \sqsubseteq \exists r.B\in \Tmc$
	& {\bf then} add $(A,r,B)$ to \Smc \\
${\bf C_3}$ & {\bf if} $\{(X,r,Y),(Y,A)\}\subseteq \Smc$ 
	& {\bf and} $\exists r.A\sqsubseteq B\in \Tmc$
	& {\bf then} add $(X,B)$ to \Smc \\
\midrule	
${\bf L_1}$ & {\bf if} & \phantom{and} $(A\mid B)[\ell,1]\in\Cmc$ 
	& {\bf then} $\Lmc(A)\gets \ell\cdot\Lmc(B)$ \\
${\bf L_2}$ & {\bf if} & \phantom{and} $A_1\sqcap A_2\sqsubseteq B\in \Tmc$
	& {\bf then} $\Lmc(B)\gets \Lmc(A_1)+\Lmc(A_2)-1$ \\
${\bf L_3}$ & {\bf if} $(B,A)\in \Smc$ & 
	& {\bf then} $\Lmc(A)\gets \Lmc(B)$ \\
\bottomrule
\end{tabular}
\end{table}
In each case, a rule is only applied if its execution extends the available knowledge; that is,
if either \Smc is extended to include one more tuple, or a lower bound in \Lmc is increased.
In the latter case, only the larger value is kept through the function \Lmc.

The first three rules in Table~\ref{tab:rules} 
are the standard completion rules for classical \el. The remaining rules update
the lower bounds for the likelihood of all relevant concept names, taking into account their
logical relationship, as explained next.

Rule \rL1 applies the obvious inference associated to conditional statements: from all the 
individuals that belong to $B$, $(A\mid B)[\ell,1]$ states that at least $100\ell\%$ belong also to $A$.
Thus, assuming that $\Lmc(B)$ is the lowest proportion of elements in $B$ possible, the proportion of 
elements in $A$ must be at least $\ell\cdot\Lmc(B)$.
\rL3 expresses that if every element of $B$ must also belong to $A$, then there must be at least as
many elements in $A$ as there are in $B$.
Finally, \rL2 deals with the fact that two concepts that are proportionally large must necessarily
overlap. For example, if $60\%$ of all individuals belong to $A$ and $50\%$ belong to $B$, then 
at least $10\%$ must belong to both $A$ and $B$; otherwise, together they would cover more than 
the whole domain.

The algorithm executes all the rules until \emph{saturation}; that is, until no rule is applicable. 
Once it is saturated, we can decide $m$-necessity from the function \Lmc as follows:
$A$ is $m$-necessary iff $m\le \Lmc(A)$.
Before showing the correctness of this algorithm, we show an important property.

Notice that the likelihood information from \Lmc is never transferred through roles. The reason for this
is that an existential restriction $\exists r.B$ only guarantee the existence of one element belonging to
the concept $B$. Proportionally, the number of elements that belong to $B$ tends to $0$.
\begin{example}
\label{exa:roles}
Consider the KB $(\{\top\sqsubseteq \exists r.A\},\emptyset)$. For any $n\in\mathbb{N}$,
construct the interpretation $\Imc_n:=(\{0,\ldots,n\},\cdot^{\Imc_n})$, where 
$A^{\Imc_n}=\{0\}$ and $r^{\Imc_n}=\{(k,0)\mid 0\le k\le n\}$.
It is easy to see that $\Imc_n$ is a model of the KB and 
$\nicefrac{[A]^{\Imc_n}}{[\top]^{\Imc_n}}<\nicefrac{1}{n}$. Thus, the best lower bound for
$m$-necessity of $A$ is $0$, as correctly given by the algorithm.
\end{example}
\begin{theorem}[correctness]
Let \Lmc be the function obtained by the application of the rules until saturation and $A_0\in N_C$. 
Then $A_0$ is $m$-necessary iff $m\le \Lmc(A)$.
\end{theorem}
\begin{proof}[sketch]
It is easy to see that all the rules are sound, which proves the ``if'' direction. For the converse
direction, we consider a finite domain $\Delta$ and an interpretation $\cdot^\Imc$ of the concept names 
such that $\nicefrac{[A]^\Imc}{|\Delta|}=\Lmc(A)$ and the post-conditions of the rules \rL1--\rL3
are satisfied. Such interpretation can be obtained recursively by considering the last rule application that 
updated $\Lmc(A)$. 
Assume w.l.o.g.\ that the domain is large enough so that $c/|\Delta|<m-\Lmc(A_0)$, where
$c$ is the number of concept names appearing in \Kmc.
It is easy to see that this interpretation satisfies all conditional statements and the GCIs
$A_1\sqcap A_2\sqsubseteq B\in\Tmc$. For every concept name $A$, create a new domain element
$\delta_A$ and extend the interpretation \Imc such that $\delta_A\in B$ iff $(A,B)\in \Smc$.
Given a role name $r$, we define
$r^\Imc := \{(\gamma,\delta_B)\mid A\sqsubseteq \exists r.B, \gamma\in A^\Imc\}$. Then,
this interpretation satisfies the KB \Kmc, and $[A_0]^\Imc/|\Delta|\le \Lmc(A_0) + c/|\Delta|<m$.
\qed
\end{proof}
Thus, the algorithm can correctly decide $m$-necessity of a given concept name. It remains only to be
shown that the process terminates after polynomially many rule applications.
To guarantee this, we impose an ordering in the rule applications. First, we apply all the classical
rules \rC1--\rC3, and only when no such rules are applicable, we update the function \Lmc
through the rules \rL1--\rL3. In this case, the rule that will update to the largest possible value is
applied first.
It is known that only polynomially many classical rules (on the size of \Tmc) can be applied~\cite{BaBL05}. 
Deciding which bound rule to apply next requires polynomial time on the number of concept names
in \Kmc. Moreover, since the largest update is applied first, the value of $\Lmc(A)$ is changed at most
once for every concept name $A$. Hence, only linearly many rules are applied.
Overall, this means that the algorithm terminates after polynomially many rule applications, which
yields the following result.

\begin{theorem}
Deciding $m$-necessity is in {\sc P}.
\end{theorem}

\section{Related Work}

Over the years, various probabilistic extensions of description logics have been investigated, see, for instance, 
\cite{koller1997p,LuSt-JWS08,niepert2011log,klinov2013pronto,CePe17,riguzzi2015probabilistic,penaloza2016probabilistic}.
The one that is closest to our approach is the type 1 extension of \ALC proposed in the appendix 
of~\cite{LutzS10}. Briefly,~\cite{LutzS10} introduces probabilistic constraints of the form 
$P(C\mid D) \leq p$, $P(C\mid D) = p$, $P(C\mid D) \geq p$ for \ALC concepts $C,D$. 
These correspond to the conditionals $(C\mid D)[0,p]$, $(C\mid D)[p,p]$, $(C\mid D)[p,1]$, respectively.
Conversely, each conditional can be rewritten as such a probabilistic constraint.  
However, there is a subtle but fundamental difference in the semantics. While the definition
in~\cite{LutzS10} allows for probability distributions over arbitrary domains,
we do not consider uncertainty over the domain. This comes down to allowing only finite domains
and only the uniform distribution over this domain; that is, our approach further restricts the class of models of a KB.
One fundamental difference between the two approaches is that 
Proposition~\ref{prop:inclusion_and_det_conditionals} does not hold in~\cite{LutzS10}: the reason is that 
the conditional $(C\mid D)[1,1]$ can be satisfied by an interpretation \Imc that contains an element
$x\in (C\sqcap \neg D)^\Imc$, where $x$ has probability $0$.  

This difference is the main reason why the {\sc ExpTime} algorithm proposed by Lutz and Schr\"oder
cannot be transferred to our setting. It does not suffice to consider the satisfiable types independently,
but other implicit subsumption relations may depend on the conditionals only.
\begin{example}
Consider the statistical \el KB $\Kmc=(\Tmc,\Cmc)$ with 
\begin{align*}
\Tmc:={} & \{\top\sqsubseteq \exists r. A,\quad \exists r.B\sqsubseteq C\} \\
\Cmc:={} & \{(B\mid\top)[0.5,1], \quad (A\mid B)[0.5,1], \quad (A\mid\top)[0,0.25] \}
\end{align*}
From \Cmc it follows that every element of $A$ must also belong to $B$, and hence every domain
element must be an element of $C$. However, $\neg C$ defines a satisfiable type
(w.r.t.\ \Tmc) which will be interpreted as non-empty in the model generated by the approach
in~\cite{LutzS10}.
\end{example}

\section{Conclusions}

We have introduced Statistical \ALC, a new probabilistic extension of the description logic \ALC
for statistical reasoning. 
We analyzed the basic properties of this logic and 
introduced some reasoning problems that we are interested in. 
As a first step towards effective reasoning in Statistical \ALC, 
we focused on \el,
a well-known sublogic of \ALC that, in its classical form, allows for polynomial-time reasoning.
We showed that upper bounds in conditional constraints  make 
the satisfiability problem in statistical \el\ {\sc NP}-hard and gave an {\sc NExpTime} algorithm
to decide satisfiability. We showed that
tractability can be regained by disallowing strict upper bounds in
the conditional statements.

We are going to provide more algorithms and a more complete picture of the complexity of reasoning 
for Statistical \ALC and its fragments in future work. A combination of integer programming and
the inclusion-exclusion principle may be fruitful to design first algorithms for reasoning in full   
Statistical \ALC.

\bibliographystyle{splncs03}
\bibliography{references}

\end{document}

%% file: sec_properties.tex
\section{Statistical \ALC}
\label{sec_stat_alc}

We start by revisiting the classical description logic \ALC.
Given two disjoint sets \NC of \emph{concept names} and \NR of \emph{role names}, 
\ALC concepts are built using the grammar rule
$
C ::= \top \mid A \mid \neg C \mid C\sqcap C\mid \exists r.C,
$
where $A\in\NC$ and $r\in\NR$.
One can express disjunction, universal quantification
and subsumption through the usual logical equivalences
like $C_1 \sqcup C_2 \equiv \neg (\neg C_1 \sqcap \neg C_2)$.
For the semantics, we focus on finite interpretations. An
\emph{\ALC interpretation} $\Imc = (\Delta^\Imc, \cdot^\Imc)$ consist of a 
non-empty, finite domain $\Delta^\Imc$ and
an interpretation function $\cdot^\Imc$ that maps concept names $A\in\NC$ to sets 
$A^\Imc \subseteq \Delta^\Imc$ and roles names $r\in\NR$ to binary relations 
$r^\Imc \subseteq \Delta^\Imc \times \Delta^\Imc$.
Two \ALC\ concepts $C_1, C_2$  are \emph{equivalent} ($C_1 \equiv C_2$) iff $C_1^\Imc = C_2^\Imc$ 
for all interpretations $\Imc$.

Here, we consider a probabilistic extension of $\ALC$.
Statistical \ALC knowledge bases consist of probabilistic conditionals that are built up
over \ALC concepts.
\begin{definition}[Conditionals, Statistical KB]
A \emph{probabilistic $\ALC$ conditional} is an expression of the form
$(C\mid D)[\ell,u]$, where $C,D$ are \ALC concepts and $\ell, u \in \mathbb{Q}$ are rational numbers such that $0\le \ell\le u\le 1$.
A \emph{statistical \ALC knowledge base} (KB) is a set $\kb$ of probabilistic $\ALC$ conditionals.
\end{definition}
For brevity, we usually call probabilistic \ALC conditionals simply \emph{conditionals}.
\begin{example}
\label{flu_example}
Let $\kb_\textit{{flu}} = \{(\exists \conc{has}.\conc{fever}\mid \exists \conc{has}.\conc{flu})[0.9,0.95],
(\exists \conc{has}.\conc{flu} \mid \top)[0.01,0.03]\}$.
$\kb_{\textit{{flu}}}$ states that $90$ to $95$ percent of patients who have the flu
have fever, and that only $1$ to $3$ percent of patients have the flu. 
\end{example}
Intuitively, a conditional $(C\mid D)[\ell,u]$ expresses that the relative proportion of elements of $D$ 
that also belong to $C$ is between $\ell$ and $u$. 
In order to make this more precise,
consider a finite \ALC interpretation \Imc, and an \ALC concept $X$. We denote the cardinality of $X^\Imc$
by $\countDom{X}$, that is, $\countDom{X}:=|X^\Imc|$.
The interpretation \Imc \emph{satisfies} $(C\mid D)[\ell,u]$, written as $\Imc \models (C\mid D)[\ell,u]$, iff either 
$\countDom{D}=0$ or 
\begin{equation}
\frac{\countDom{C\sqcap D}}{\countDom{D}}\in[\ell,u].
\label{eq:defCondSat}
\end{equation}
\Imc \emph{satisfies} a statistical \ALC knowledge base $\kb$ iff it satisfies all conditionals in $\kb$.
In this case, we call $\Imc$ a model of $\kb$ and write $\Imc \models \kb$. 
We denote the set of all models of $\kb$ by $\modelSet(\kb)$.
As usual, $\kb$ is \emph{consistent} if $\modelSet(\kb) \neq \emptyset$ and inconsistent otherwise.
We call two knowledge bases $\kb_1, \kb_2$ equivalent and write $\kb_1 \equiv \kb_2$ iff
$\modelSet(\kb_1) = \modelSet(\kb_2)$.
\begin{example}
\label{flu_example_model}
Consider again the KB $\kb_{\textit{{flu}}}$ from Example \ref{flu_example}.
Let $\Imc$ be an interpretation with $1000$ individuals. $10$ of these have the flu and
$9$ have both the flu and fever. Then $\Imc \in \modelSet(\kb_{\textit{{flu}}})$.
\end{example}

In classical \ALC, knowledge bases are defined by a set of general concept inclusions (GCIs) 
$C \sqsubseteq D$ that express that $C$ is a subconcept of $D$. An interpretation $\Imc$ satisfies 
$C \sqsubseteq D$ iff $C^\Imc \subseteq D^\Imc$.
As shown next, GCIs can be seen as a special kind of conditionals, and hence statistical \ALC KBs are 
a generalization of classical \ALC KBs.
\begin{proposition}
\label{prop:inclusion_and_det_conditionals}
For all statistical \ALC interpretations $\Imc$, we have
$\Imc \models C \sqsubseteq D$ iff $\Imc \models (D\mid C)[1,1]$.
\end{proposition}
\begin{proof}
If $\Imc \models C \sqsubseteq D$ then $C^\Imc \subseteq D^\Imc$ and 
$C^\Imc \cap D^\Imc = C^\Imc$. If $C^\Imc = \emptyset$, we have $\countDom{C}=0$.
Otherwise $\frac{\countDom{C\sqcap D}}{\countDom{C}}=1$. Hence, $\Imc \models (D\mid C)[1,1]$.

Conversely, assume  $\Imc \models (D\mid C)[1,1]$. If $\countDom{C}=0$, then
$C^\Imc = \emptyset$ and $\Imc \models C \sqsubseteq D$. Otherwise,
$\frac{\countDom{C\sqcap D}}{\countDom{C}}=1$, that is, $\countDom{C\sqcap D} = \countDom{C}$. 
If there was a $d \in C^\Imc \setminus D^\Imc$, we had $\countDom{C\sqcap D} < \countDom{C}$,
hence, we have  $C^\Imc \subseteq D^\Imc$  and $\Imc \models C \sqsubseteq D$.
\qed
\end{proof}
Given a statistical \ALC knowledge base $\kb$, the first problem that we are interested in is deciding consistency
of $\kb$. We define the satisfiability problem for statistical \ALC knowledge bases as usual.
\paragraph{\bf Satisfiability Problem:} Given a knowledge base $\kb$, decide whether $\modelSet(\kb) \neq \emptyset$.
\begin{example}
\label{flu_example_sat}
Consider again the knowledge base $\kb_{\textit{{flu}}}$ from Example \ref{flu_example}.
The conditional $(\exists \conc{has}.\conc{flu} \mid \top)[0.01,0.03]$ implies that
$\countDom{\exists \conc{has}.\conc{flu}} \geq 0.01$ for all models $\Imc \in \modelSet(\kb_{\textit{{flu}}})$.
$(\exists \conc{has}.\conc{fever}\mid \exists \conc{has}.\conc{flu})[0.9,0.95]$
implies 
$\countDom{\exists \conc{has}.\conc{fever}\sqcap \exists \conc{has}.\conc{flu}} 
\geq
0.9 \countDom{\exists \conc{has}.\conc{flu}} $.
Therefore, 
$\countDom{\exists \conc{has}.\conc{fever}} \geq
\countDom{\exists \conc{has}.\conc{fever}\sqcap \exists \conc{has}.\conc{flu}} \geq
0.9 \countDom{\exists \conc{has}.\conc{flu}} \geq 0.009$.
Hence, adding the conditional 
$(\exists \conc{has}.\conc{fever} \mid \top)[0,0.005]\}$
renders $\kb_{\textit{{flu}}}$ inconsistent.
\end{example}
If $\kb$ is consistent, we are interested in deriving (implicit) probabilistic conclusions.
We can think of different reasoning problems in this context. 
First, we can define an entailment relation analogously to logical entailment.
Then, the probabilistic conditional $(C\mid D)[\ell,u]$ is an \emph{l-consequence} of the KB
$\kb$ iff $\modelSet(\kb) \subseteq \modelSet(\{(C\mid D)[\ell,u]\})$. In this case, we write
$\kb \models_l (C\mid D)[\ell,u]$.
In the context of type 2 probabilistic conditionals, this entailment relation has also been called
 just \emph{logical consequence} \cite{Lukasiewicz2001}.
\paragraph{\bf l-Entailment Problem:} Given a knowledge base $\kb$ and a conditional 
	$(C\mid D)[\ell,u]$, decide whether $\kb \models_l (C\mid D)[\ell,u]$.
\begin{example}
\label{flu_example_l_ent}
Consider again the KB $\kb_{\textit{{flu}}}$ from Example \ref{flu_example}.
As explained in Example \ref{flu_example_sat}, $\countDom{\exists \conc{has}.\conc{fever}}\geq 0.009$
holds for all models $\Imc \in \modelSet(\kb)$.
Therefore, it follows that $\kb_{\textit{{flu}}} \models_l (\exists \conc{has}.\conc{fever} \mid \top)[0.009,1]$.
That is, our statistical information suggests that at least $9$ out of $1,000$ of our patients have fever.
\end{example}
\begin{example}
\label{penguin_example}
Consider a domain with birds (B), penguins (P) and
flying animals (F).
We let 
$\kb_{\textit{{birds}}} = \{(B \mid \top)[0.5,0.6], (F \mid B)[0.85,0.9],(F \mid P)[0,0]\}.$
Note that the conditional $(F \mid B)[0.85,0.9]$ is actually equivalent to 
$(\neg F \mid B)[0.1,0.15]$. Furthermore, for all $\Imc \in \modelSet(\kb_{\textit{{birds}}})$,
$(F \mid P)[0,0]$ implies $\countDom{P \sqcap F} = 0$. 
Therefore, we have 
$\countDom{P \sqcap B} = \countDom{B \sqcap P \sqcap F} + \countDom{B \sqcap P \sqcap \neg F}
\leq 0 +  \countDom{B \sqcap \neg F}
\leq 0.15  \countDom{B}$.
Hence, $\kb_{\textit{{birds}}} \models_l (P \mid B)[0,0.15]$.
That is, our statistical information suggests that at most $15$ out of $100$ birds in our population are penguins.
\end{example}
As usual, the satisfiability problem can be reduced to the l-entailment problem.
\begin{proposition}
$\kb$ is inconsistent iff $\kb \models_l (\top \mid \top)[0,0]$.
\end{proposition}
%
\begin{proof}
If $\kb$ is inconsistent, then $\modelSet(\kb) = \emptyset$
and so $\kb \models_l (\top \mid \top)[0,0]$.

Conversely, assume $\kb \models_l (\top \mid \top)[0,0]$.
We have $\countDom{\top} > 0$ and $\frac{\countDom{\top \sqcap \top}}{\countDom{\top}} = 1$
for all interpretations $\Imc$.
Hence, $\modelSet(\{(\top \mid \top)[0,0]\}) = \emptyset$ 
and since $\kb \models_l (\top \mid \top)[0,0]$, we must have $\modelSet(\kb) = \emptyset$  as well.
\qed
\end{proof}
%
Often, we do not want to check whether a specific conditional is entailed, but rather deduce tight
probabilistic bounds for a statement. This problem is often referred to as the 
\emph{probabilistic entailment problem} in other probabilistic logics, see \cite{Nilsson1986AI,HansenJaumard2000,Lukasiewicz2001} for instance.
Consider a \emph{query} of the form $(C\mid D)$, where $C,D$ are \ALC concepts.
We define the p-Entailment problem similar to the probabilistic entailment problem for type 2 probabilistic logics.
\paragraph{\bf p-Entailment Problem:} Given knowledge base $\kb$ and a query $(C\mid D)$,
	 find minimal and maximal solutions of the optimization problems
	\begin{alignat*}{2}
	&\inf_{\Imc \in \modelSet(\kb)} / \sup_{\Imc \in \modelSet(\kb)} \quad
		 &&\frac{\countDom{C\sqcap D}}{\countDom{D}} \\
	&\textit{subject to} &&\countDom{D} > 0
	\end{alignat*}
Since the objective function $\frac{\countDom{C\sqcap D}}{\countDom{D}}$ is bounded from below by $0$ and
from above by $1$, the infimum $m$ and the maximum $M$ are well-defined whenever there is a model 
$\Imc \in \modelSet(\kb)$ such that $\countDom{D} > 0$. 
In this case, we say that $\kb$ p-entails
$(C\mid D)[m,M]$ and write $\kb \models_p (C\mid D)[m,M]$.
In the context of type 2 probabilistic conditionals, this entailment relation has also been called
\emph{tight logical consequence} \cite{Lukasiewicz2001}.
If $\countDom{D} = 0$ for all $\Imc \in \modelSet(\kb)$, the p-Entailment problem is infeasible, that is,
there exists no solution. 
\begin{example}
\label{penguin_example_2}
In Example \ref{penguin_example}, we found that $\kb_{\textit{{birds}}} \models_l (P \mid B)[0,0.15]$.
This bound is actually tight. Since $0$ is always a lower bound and we showed that $0.15$ is an upper bound, 
it suffices to give examples of interpretations that take these bounds. 
For the lower bound, let $\Imc_0$ be an interpretation with
$200$ individuals. $100$ of these individuals are birds and $85$ are birds that can fly. There are no penguins.
Then $\Imc_0$ is a model of $\kb_{\textit{{birds}}}$ with $\countDomP{B}{\Imc_0} > 0$ that satisfies $(P \mid B)[0,0]$.
Construct $\Imc_1$ from $\Imc_0$ by letting the 15 non-flying birds be penguins.
Then $\Imc_1$ is another model of $\kb_{\textit{{birds}}}$ and $\Imc_1$ satisfies $(P \mid B)[0.15,0.15]$.
Hence, we also have $\kb_{\textit{{birds}}} \models_p (P \mid B)[0,0.15]$.
\end{example}

If $\kb \models_p (C\mid D)[m,M]$, one might ask whether the values between $m$ and $M$ are actually taken
by some model of $\kb$ or whether there can be large gaps in between. 
For the probabilistic entailment problem for type 2 logics, we can show that the models of $\kb$ do indeed 
yield a dense interval by noting that each convex combination of models is a model and applying the Intermediate Value Theorem from Real Analysis. However, in our framework, we do not consider probability distributions over possible worlds, but the worlds themselves, which are discrete in nature. We therefore cannot apply the same 
tools here.
However, for each two models that yield different probabilities for a query, 
we can find another model that takes the probability in the middle of these probabilities.  
\begin{lemma}[Bisection Lemma]
Let $C, D$ be two arbitrary \ALC concepts.
If there exist $\Imc_0, \Imc_1 \in \modelSet(\kb)$ such that 
$r_0 = \frac{\countDomP{C\sqcap D}{\Imc_0}}{\countDomP{D}{\Imc_0}} <
\frac{\countDomP{C\sqcap D}{\Imc_1}}{\countDomP{D}{\Imc_1}} = r_1$,
then there is an $\Imc_{0.5} \in \modelSet(\kb)$ such that
$\frac{\countDomP{C\sqcap D}{\Imc_{0.5}}}{\countDomP{D}{\Imc_{0.5}}} = \frac{r_0+ r_1}{2}$.
\end{lemma}
\ifnum\ProofVersion=1
\begin{proof}
Given an interpretation $\Imc$ and $n \in \mathbb{N}$, we construct the interpretation $\Imc^{(n)}$
as follows. We set $\Delta^{\Imc^{(n)}} = \{d_1, \dots, d_n \mid d \in \Delta^\Imc\}$; that is, we make
$n$ different copies of the domain. 
For all $A \in \NC$, we set $A^{\Imc^{(n)}} = \{d_1, \dots, d_n \mid d \in A^\Imc\}$, and for all
$r \in \NR$, we set 
$r^{\Imc^{(n)}} = \{(d_1,e_1), \dots, (d_1,e_n), \dots (d_n,e_1), \dots, (d_n,e_n) \mid (d,e) \in r^\Imc\}$.
By induction on the shape of \ALC concepts, we can show that 
$\countDom{F} = n \countDomP{F}{\Imc^{(n)}}$
for all concepts $F$. 

Let now $(C_i\mid D_i)[\ell_i,u_i]$, $i=1,\dots,|\kb|$ be all the conditionals from $\kb$.
Let $\ell$ be the least common multiple of all values from 
$\countDomP{D}{\Imc_0}, \countDomP{D}{\Imc_1}, \countDomP{D_1}{\Imc_0}, 
  \countDomP{D_1}{\Imc_1}, \dots, \countDomP{D_n}{\Imc_0}$, $\countDomP{D_n}{\Imc_1}$ 
  that are non-zero, and
$k, K, k_1, K_1, \dots,  k_n, K_n$ be such that $k \countDomP{D}{\Imc_0} = \ell$, 
$K \countDomP{D}{\Imc_1} = \ell,\dots, k_n \countDomP{D_n}{\Imc_0} = \ell$, 
$K_n \countDomP{D_n}{\Imc_1} = \ell$.
Assume w.l.o.g. that $\Imc_0$ and $\Imc_1$ have different domains (just rename the elements of
one domain if necessary).
For $n,N \in \mathbb{N}$, let $\Imc_{n,N}$ be the interpretation that is obtained from $\Imc_0^{(n)}$
and $\Imc_1^{(N)}$ by taking the union of the domains, concept and role interpretations. That is,
$\Delta^{\Imc_{n,N}} = \Delta^{\Imc_0^{(n)}} \cup \Delta^{\Imc_1^{(N)}}$,
$A^{\Imc_{n,N}} = A^{\Imc_0^{(n)}} \cup A^{\Imc_1^{(N)}}$  and
$r^{\Imc_{n,N}} = r^{\Imc_0^{(n)}} \cup r^{\Imc_1^{(N)}}$.
Consider $k' = k \prod_{i=1}^n k_i$, $k'_{-j} = k \prod_{i\neq j} k_i$, $K' = K \prod_{i=1}^n K_i$
and $K'_{-j} = K \prod_{i\neq j}^n K_i$. 
Then, for $i=1,\dots,|\kb|$,
\begin{align*}
\frac{\countDomP{C_i\sqcap D_i}{\Imc_{k'n,K'N}}}{\countDomP{D_i}{\Imc_{k'n,K'N}}}
&= \frac{\countDomP{C_i\sqcap D_i}{\Imc_0^{(k'n)}} + \countDomP{C_i\sqcap D_i}{\Imc_1^{(K'N)}}}{\countDomP{D_i}{\Imc_0^{(k'n)}} + \countDomP{D_i}{\Imc_1^{(K'N)}}} \\
&= \frac{k'n \countDomP{C_i\sqcap D_i}{\Imc_0} + K'N \countDomP{C_i\sqcap D_i}{\Imc_1}}{k'n \countDomP{D_i}{\Imc_0} + K'N \countDomP{D_i}{\Imc_1}}\\
&= \frac{k'n \countDomP{C_i\sqcap D_i}{\Imc_0} + K'N \countDomP{C_i\sqcap D_i}{\Imc_1}}{(k'_{-i} n + K'_{-i} N)l}\\
&= \frac{k'n }{k'_{-i} n + K'_{-i} N} \frac{\countDomP{C_i\sqcap D_i}{\Imc_0}}{k_i  \countDomP{D_i}{\Imc_0}}
+ \frac{K'N }{k'_{-i} n + K'_{-i} N} \frac{\countDomP{C_i\sqcap D_i}{\Imc_1}}{K_i \countDomP{D_i}{\Imc_1}}\\
&= \frac{k'n }{k'_{-i} n + K'_{-i} N} \frac{\countDomP{C_i\sqcap D_i}{\Imc_0}}{ \countDomP{D_i}{\Imc_0}}
+ \frac{K'_{-i}N }{k'_{-i} n + K'_{-i} N} \frac{\countDomP{C_i\sqcap D_i}{\Imc_1}}{ \countDomP{D_i}{\Imc_1}}.
\end{align*}
The last equality shows that $\frac{\countDomP{C_i\sqcap D_i}{\Imc_{k'n,K'N}}}{\countDomP{D_i}{\Imc_{k'n,K'N}}}$ is a convex combination of $\frac{\countDomP{C_i\sqcap D_i}{\Imc_0}}{ \countDomP{D_i}{\Imc_0}}$ and
$\frac{\countDomP{C_i\sqcap D_i}{\Imc_1}}{ \countDomP{D_i}{\Imc_1}}$. Since, $\Imc_0$ and $\Imc_1$ satisfy
the $i$-th conditional, $\Imc_{k'n,K'N}$ satisfies the conditional as well. In case that both 
$\countDomP{D_i}{\Imc_0}=0$ and $\countDomP{D_i}{\Imc_1}=0$, we have $\countDomP{D_i}{\Imc_{k'n,K'N}}=0$ as well and so the
conditional is still satisfied. If only $\countDomP{D_i}{\Imc_0}=0$, we can see from the second inequality that
$\frac{\countDomP{C_i\sqcap D_i}{\Imc_{k'n,K'N}}}{\countDomP{D_i}{\Imc_{k'n,K'N}}} = 
\frac{k'n 0 + K'N \countDomP{C_i\sqcap D_i}{\Imc_1}}{k'n 0 + K'N \countDomP{D_i}{\Imc_1}}
= \frac{\countDomP{C_i\sqcap D_i}{\Imc_1}}{\countDomP{D_i}{\Imc_1}}
$ and the conditional is still satisfied. The case $\countDomP{D_i}{\Imc_1}=0$ is analogous of course.
Hence, $\Imc_{k'n,K'N} \in \modelSet(\kb)$ for all choices of $n$ and $N$.

Let $k_0 = \prod_{i=1}^n k_i$ and $K_0 = \prod_{i=1}^n K_i$.
Then we can show completely analogously that 
$\frac{\countDomP{C\sqcap D}{\Imc_{k'n,K'N}}}{\countDomP{D}{\Imc_{k'n,K'N}}}
=\frac{k_0 n }{k_0 n  + K_0  N } \frac{\countDomP{C_i\sqcap D_i}{\Imc_0}}{ \countDomP{D_i}{\Imc_0}}
+ \frac{K_0 N  }{ k_0 n  + K_0  N } \frac{\countDomP{C_i\sqcap D_i}{\Imc_1}}{ \countDomP{D_i}{\Imc_1}}.
$
Letting $n=K'$ and $N=K'$, we have
$\frac{\countDomP{C\sqcap D}{\Imc_{k'K',K'k'}}}{\countDomP{D}{\Imc_{k'K',K'k'}}}
=\frac{1}{2} \frac{\countDomP{C_i\sqcap D_i}{\Imc_0}}{ \countDomP{D_i}{\Imc_0}}
+ \frac{1}{2} \frac{\countDomP{C_i\sqcap D_i}{\Imc_1}}{ \countDomP{D_i}{\Imc_1}}
=\frac{r_0+ r_1}{2}.
$
\qed
\end{proof}
\fi
We can now show that for each value between the lower and upper bound given by p-entailment, we can
find a model that gives a probability arbitrarily close to this value.
\begin{proposition}[Intermediate Values]
\label{prop:intermed_values}
Let $\kb \models_p (C\mid D)[m,M]$. Then for every $x \in (m,M)$ (where $(m,M)$ denotes the open 
interval between $m$ and $M$) and for all $\epsilon >0$, there
is a $\Imc_{x,\epsilon} \in \modelSet(\kb)$ such that $|\frac{\countDomP{C\sqcap D}{\Imc_{x,\epsilon}}}{\countDomP{D}{\Imc_{x,\epsilon}}} - x| < \epsilon$.
\end{proposition}
\ifnum\ProofVersion=1
\begin{proof}
Since $\kb \models_p (C\mid D)[m,M]$, there must exist an $\Imc_0 \in \modelSet(\kb)$ such that 
$m \leq \frac{\countDomP{C\sqcap D}{\Imc_0}}{\countDomP{D}{\Imc_0}} \leq x$
and  an $\Imc_1 \in \modelSet(\kb)$ such that 
$x \leq \frac{\countDomP{C\sqcap D}{\Imc_1}}{\countDomP{D}{\Imc_1}} \leq M$.

Consider the following bisection algorithm: we let 
$\Imc^\bot_0 = \Imc_0$,
$\Imc^\top_0 = \Imc_1$.
Then starting from $i=1$, we let $\Imc^{0.5}_i$ be the model of $\kb$ that is obtained
from  $\Imc^\bot_{i-1}$ and $\Imc^\top_{i-1}$ as explained in the bisection lemma. 
If $\frac{\countDomP{C\sqcap D}{\Imc^{0.5}_i}}{\countDomP{D}{\Imc^{0.5}_i}} = x$,
we are done. Otherwise, if 
$\frac{\countDomP{C\sqcap D}{\Imc^{0.5}_i}}{\countDomP{D}{\Imc^{0.5}_i}} < x$,
we let $\Imc^\bot_i = \Imc^\bot_{i-1}$ and $\Imc^\top_i = \Imc^{0.5}$.
Otherwise, we have $\frac{\countDomP{C\sqcap D}{\Imc^{0.5}_i}}{\countDomP{D}{\Imc^{0.5}_i}} < x$,
and we let $\Imc^\bot_i = \Imc^{0.5}_i$ and $\Imc^\top_i = \Imc^\top_{i-1}$.
By construction, we maintain the invariant 
$\frac{\countDomP{C\sqcap D}{\Imc^\bot_i}}{\countDomP{D}{\Imc^\bot_i}} \leq x
\leq \frac{\countDomP{C\sqcap D}{\Imc^\top_i}}{\countDomP{D}{\Imc^\top_i}}$
and we have 
$\frac{\countDomP{C\sqcap D}{\Imc^\top_i}}{\countDomP{D}{\Imc^\top_i}}
- \frac{\countDomP{C\sqcap D}{\Imc^\bot_i}}{\countDomP{D}{\Imc^\bot_i}} 
\leq \frac{M-m}{2^n}
$.
Hence, after at most $i = \big\lceil \log\big(\frac{M-m}{\epsilon} \big) \big\rceil$ iterations, 
$\Imc^{0.5}_i$ is a model of $\kb$ that proves the claim. 
\qed
\end{proof}
\fi

\section{Logical Properties}
\label{sec_properties}

We now discuss some logical properties of Statistical \ALC.
We already noted that Statistical \ALC generalizes classical \ALC in Proposition 
\ref{prop:inclusion_and_det_conditionals}.
Furthermore, p-entailment yields a tight and dense (Proposition \ref{prop:intermed_values}) 
answer interval for all queries whose condition can be satisfied by models of the knowledge base. 
Let us also note that statistical \ALC is \emph{language invariant}. That is, increasing the language
by adding new concept or role names does not change the semantics of \ALC.
This can be seen immediately by observing that the interpretation of conditionals in \eqref{eq:defCondSat}
depends only on the concept and role names that appear in the conditional.

Statistical \ALC is also \emph{representation invariant} in the sense that for all concepts $C_1, D_1$ 
and $C_2,D_2$, if  $C_1 \equiv C_2$ and $D_1 \equiv D_2$
then $(C_1 \mid D_1)[l,u] \equiv (C_2 \mid D_2)[l,u]$.
Hence, changing the syntactic representation 
of conditionals does not change their semantics. 
In particular, entailment results are independent of such changes.

Both l- and p-entailment satisfy the following \emph{independence} property: 
whether or not $\kb \models_l (C\mid D)[\ell,u]$ ($\kb \models_p (C\mid D)[m,M]$) depends only
on the conditionals in $\kb$ that are connected with the query. This may simplify answering the query
by reducing the size of the KB.
In order to make this more precise, we need some additional definitions. For an arbitrary \ALC concept 
$C$, $\sig(C)$ denotes the set of all concept and role names appearing in $C$.
The conditionals $(C_1\mid D_1)[\ell_1,u_1]$ and $(C_2\mid D_2)[\ell_2,u_2]$ are 
\emph{directly connected}
(written $(C_1\mid D_1)[\ell_1,u_1] \rightleftharpoons (C_2\mid D_2)[\ell_2,u_2]$)
if and only if $(\sig(C_1) \cup \sig(D_1)) \cap (\sig(C_2) \cup \sig(D_2)) \neq \emptyset$. That is, two conditionals
are directly connected iff they share concept or role names.  
Let $\rightleftharpoons^*$ denote the transitive closure of $\rightleftharpoons$.
We say that $(C_1\mid D_1)[\ell_1,u_1]$ and $(C_2\mid D_2)[\ell_2,u_2]$ are \emph{connected} iff 
$(C_1\mid D_1)[\ell_1,u_1] \rightleftharpoons^* (C_2\mid D_2)[\ell_2,u_2]$.
The \emph{restriction of $\kb$ to conditionals connected to $(C\mid D)[\ell,u]$}
is the set $\{\kappa \in \kb \mid \kappa \rightleftharpoons^* (C\mid D)[\ell,u]\}$.
Using an analogous definition for queries (qualitative conditionals) $(C_1\mid D_1)$ and $(C_2\mid D_2)$,
we get the following result.
\begin{proposition}[Independence] If $\kb$ is consistent, we have
\begin{enumerate}
	\item $\kb \models_l (C\mid D)[\ell,u]$ iff  
	 $\{\kappa \in \kb \mid \kappa \rightleftharpoons^* (C\mid D)[\ell,u]\}  \models_l (C\mid D)[\ell,u]$.
	\item $\kb \models_p (C\mid D)[m,M]$ iff  
	 $\{\kappa \in \kb \mid \kappa \rightleftharpoons^* (C\mid D)\}  \models_p (C\mid D)[m,M]$.
\end{enumerate}
\end{proposition}
\ifnum\ProofVersion=1
\begin{proof}
For both claims, it suffices to show that for each model $\Imc_1$ of $\kb$, 
there is a model $\Imc_2$ of $\{\kappa \in \kb \mid \kappa \rightleftharpoons^* (C\mid D)\}$ 
($\{\kappa \in \kb \mid \kappa \rightleftharpoons^* (C\mid D)[\ell,u]\}$)
such that $\countDomP{D}{\Imc_1} = \countDomP{D}{\Imc_2}$
and $\countDomP{C\sqcap D}{\Imc_1} = \countDomP{C \sqcap D}{\Imc_2}$ and vice versa. 

If $\Imc_1$ is a model of $\kb$, let $\Imc_2$ be the restriction of $\Imc_1$ 
to the concept and role names in $\{\kappa \in \kb \mid \kappa \rightleftharpoons^* (C\mid D)\}$. 
Then $\Imc_2$ is still a model of $\{\kappa \in \kb \mid \kappa \rightleftharpoons^* (C\mid D)\}$. In particular,
$\countDomP{D}{\Imc_1} = \countDomP{D}{\Imc_2}$ and $\countDomP{C\sqcap D}{\Imc_1} = \countDomP{C \sqcap D}{\Imc_2}$.

Conversely, let $\Imc_2$ be a model of $\{\kappa \in \kb \mid \kappa \rightleftharpoons^* (C\mid D)\}$.
By consistency of $\kb$, there is a model $\Imc_0$ of $\kb$. 
Let $\Imc_1$ be the interpretation defined as the disjoint union of
$\Imc_0$ and $\Imc_2$. 
Since $\{\kappa \in \kb \mid \kappa \rightleftharpoons^* (C\mid D)\}$ and 
$\kb \setminus \{\kappa \in \kb \mid \kappa \rightleftharpoons^* (C\mid D)\}$
do not share any concept and role names by definition of connectedness, 
$\Imc_1$ satisfies conditionals in  $\{\kappa \in \kb \mid \kappa \rightleftharpoons^* (C\mid D)\}$
iff $\Imc_2$ does and conditionals in 
$\kb \setminus \{\kappa \in \kb \mid \kappa \rightleftharpoons^* (C\mid D)\}$
iff $\Imc_0$ does. Hence, $\Imc_1$ is a model of $\kb$. In particular, it holds that
$\countDomP{D}{\Imc_1} = \countDomP{D}{\Imc_2}$ and
$\countDomP{C\sqcap D}{\Imc_1} = \countDomP{C \sqcap D}{\Imc_2}$.
\qed
\end{proof}
\fi
Another interesting property of probabilistic logics is \emph{continuity}. Intuitively, continuity states that 
minor changes in the knowledge base do not yield major changes in the derived probabilities.
However, as demonstrated by Courtney and Paris, this condition is too strong when reasoning with the 
maximum entropy model of the knowledge base \cite[p.\ 90]{Paris:1994}. 
The same problem arises for the probabilistic entailment 
problem~\cite[Example 4]{potyka2015probabilistic}. While these logics considered subjective probabilities,
the same problem occurs in our setting for statistical probabilities as we demonstrate now.
\begin{example} 
Consider the knowledge base 
$$\kb = \{(B\mid A)[0.4,0.5], (C\mid A)[0.5,0.6], (B\mid C)[1,1], (C\mid B)[1,1]\}.$$
The interpretation $\Imc = (\{a,b\}, \cdot^\Imc)$ with $A^\Imc = \{a,b\}$, $B^\Imc = C^\Imc = \{b\}$
is a model of $\kb$, i.e., $\kb$ is consistent. In particular, since $A$ is interpreted by the whole domain
of $\Imc$ we know that 
$$\kb \models_p (A\mid \top)[m,1]$$
for some $m \in [0,1]$.
As explained in Proposition \ref{prop:inclusion_and_det_conditionals}, deterministic conditionals correspond
to concept inclusions and so $(B\mid C)[1,1]$ and $(C\mid B)[1,1]$ imply that $B^{\Imc'} = C^{\Imc'}$ for all models 
$\Imc'$ of $\kb$. 
Therefore, $\frac{[B \sqcap A]^{\Imc'}}{[A]^{\Imc'}} = \frac{[C \sqcap A]^{\Imc'}}{[A]^{\Imc'}}$.
Let $\kb'$ denote the knowledge base that is obtained from $\kb$ by decreasing the upper bound of the first conditional in $\kb$ by an arbitrarily small $\epsilon > 0$. That is,
$$\kb' = \{(B\mid A)[0.4,0.5-\epsilon], (C\mid A)[0.5,0.6], (B\mid C)[1,1], (C\mid B)[1,1]\}.$$
Then the only way to satisfy the first two conditionals in $\kb'$ is by interpreting $A$ by the empty set.
Indeed, the interpretation $\Imc_\emptyset$ that interprets all concept names by the empty set is a model of
$\kb'$. So $\kb'$ is consistent and 
$$\kb' \models_p (A\mid \top)[0,0].$$ 
Hence, a minor change in the probabilities in the knowledge base can yield a severe change in the entailed
probabilities. This means that the p-entailment relation that we consider here is not continuous in this
way either.
\end{example}
As an alternative to this strong notion of continuity, Paris proposed to measure the difference between 
KBs by the Blaschke distance between their models. Blaschke continuity says that if KBs are close
with respect to the Blaschke distance, the entailed probabilities are close. Blaschke continuity is satisfied by 
some probabilistic logics under maximum entropy and probabilistic 
entailment~\cite{Paris:1994,potyka2015probabilistic}. In \cite{Paris:1994,potyka2015probabilistic}, 
probabilistic interpretations are probability distributions over a finite number of classical interpretations and 
the distance between two interpretations is the distance between the corresponding probability vectors. We 
cannot apply this definition here because we interpret conditionals by means of classical interpretations.
It is not at all clear what a reasonable definition for the distance between two classical interpretations is.
We leave the search for a reasonable topology on the space of classical interpretations for future work.

%% file: cons_nexp.tex

On the other hand, consistency can be decided in non-deterministic exponential time, through a 
reduction to integer programming. Before describing the reduction in detail, we introduce a few
simplifications.

Recall from Proposition~\ref{prop:inclusion_and_det_conditionals} that a conditionals of the form 
$(D\mid C)[1,1]$ is equivalent to the classical GCI $C\sqsubseteq D$. Thus, in the following
we will often express statistical \el KBs as pairs $\Kmc=(\Tmc,\Cmc)$, where \Tmc is a classical
TBox (i.e., a finite set of GCIs), and \Cmc is a set of conditionals.
A statistical \el KB $\Kmc=(\Tmc,\Cmc)$ is said to be in \emph{normal form} if all the GCIs in \Tmc
are of the form
\[
A_1\sqcap A_2 \sqsubseteq B, \qquad A \sqsubseteq \exists r.B, \qquad \exists r.A\sqsubseteq B
\]
and all its conditionals are of the form
\[
(A \mid B)[\ell,u]
\]
where $A,B\in N_C\cup\{\top\}$, and $r\in N_R$.
Informally, a KB is in normal form if at most one constructor is used in any GCI, and all conditionals 
are atomic (i.e., between concept names).
Every KB can be transformed to an equivalent one (w.r.t.\ the original signature)
in linear time using the normalization rules from~\cite{BaBL05}, and introducing new concept names for 
complex concepts appearing in conditionals. More precisely, we replace any  
conditional of the form $(C\mid D)[\ell,u]$ with the statement $(A\mid B)[\ell,u]$, where $A,B$ are
two fresh concept names, and extend the TBox with the axioms $A\equiv C$, and $B\equiv D$.

The main idea behind our consistency algorithm is to partition the finite domain of a model into the different 
types that they define, and use integer programming to verify that all the logical and conditional constraints
are satisfied.
Let $N_C(\Kmc)$ denote the set of all concept names appearing in the KB \Kmc.
We call any subset $\theta\subseteq N_C(\Kmc)$ a \emph{type} for \Kmc. 
Intuitively, such a type $\theta$ represents all the elements of the domain that are interpreted to
belong to all concept names $A\in\theta$ and no concept name $A\notin\theta$.
We denote as $\Theta(\Kmc)$ the set of all types of \Kmc.
To simplify the presentation, in the following we treat $\top$ as a concept name that belongs to all
types.

Given a statistical \el KB $\Kmc=(\Tmc,\Cmc)$ in normal form, we consider an integer variable $x_\theta$ 
for every type $\theta\in\Theta(\Kmc)$. These variables will express the number of domain elements that
belong to the corresponding type. In addition, $x_\top$ will be used to represent the total size of the
domain.
We build a system of linear inequalities over these variables as follows. 
First, we require that all variables have a value at least $0$, and that the sizes of all types add exactly
the size of the domain.
\begin{align}
\sum_{\theta\in\Theta(\Kmc)}x_\theta = {} & x_\top \label{eqn:con1} \\
0 \le {} & x_\theta   & \text{for all $\theta\in\Theta(\Kmc)$}
\end{align}
Then, we ensure that all the conditional statements from the KB are satisfied by adding, for each
statement $(A\mid B)[\ell,u]\in\Cmc$ the constraint
\begin{align}
\ell\cdot\sum_{B\in\theta}x_\theta \le \sum_{A,B\in\theta} x_\theta \le u\cdot \sum_{B\in\theta}x_\theta,
\label{eqn:conditional}
\end{align}
Finally, we must ensure that the types satisfy all the logical constraints introduced by the TBox. 
The GCI $A_1\sqcap A_2 \sqsubseteq B$ states that every element that belongs to both $A_1$ and $A_2$
must also belong to $B$. This means that types containing $A_1,A_2$ but excluding $B$ should not
be populated.
We thus introduce the inequality
\begin{align}
x_\theta = {} & 0 & \text{if $A_1\sqcap A_2 \sqsubseteq B\in\Tmc$, $A_1,A_2\in\theta$, and $B\notin\theta$}  \label{eqn:conlast}
\end{align}
Dealing with existential restrictions requires checking different alternatives, which we 
solve by creating different linear programs.
The GCI $A\sqsubseteq \exists r.B$ implies that, whenever there exists an element in $A$, there must
also exist at least one element in $B$. Thus, to satisfy this axiom, either $A$ should be empty
(i.e., $\sum_{A\in\theta}x_\theta =0$), or $\sum_{B\in\theta}x_\theta \ge 1$.
Hence, for every existential restriction of the form $A\sqsubseteq \exists r.B$, we define the set
\[
\Emc_{A,B} := \{\sum_{A\in\theta}x_\theta =0, \sum_{B\in\theta}x_\theta \ge 1\}
\]
To deal with GCIs of the form
$\exists r.A\sqsubseteq B$, we follow a similar approach, together with the ideas of the completion
algorithm for classical \el.
For every pair of existential restrictions $A\sqsubseteq \exists r.B, \exists r.C\sqsubseteq D$,
we define the set
\[
\Fmc_{A,B,C,D} := \{\sum_{A\in\theta,D\notin\theta}x_\theta =0, \sum_{B\in\theta,C\notin\theta}x_\theta \ge 1\}
\]
Intuitively, $\sum_{A\in\theta,D\notin\theta}x_\theta\ge 1$ whenever there exists an element that
belongs to $A$ but not to $D$. If this is the case, and the GCIs 
$A\sqsubseteq \exists r.B,\exists r.C\sqsubseteq D$ belong to the TBox \Tmc, then there must 
exist some element that belongs to $B$ but not to $C$. 

We call the hitting sets of 
$$\{\Emc_{A,B}\mid A\sqsubseteq\exists r.B\in\Tmc\}\cup
  \{\Fmc_{A,B,C,D}\mid A\sqsubseteq\exists r.B,\exists r.C\sqsubseteq D\in\Tmc\}$$
\emph{choices for \Tmc}. A \emph{program for \Kmc}  is an integer program containing all the
inequalities~\eqref{eqn:con1}--\eqref{eqn:conlast} and a choice for \Tmc.
Then we get the following result.
\begin{lemma}
\Kmc is consistent iff there exists a program for \Kmc that is satisfiable.
\end{lemma}
\begin{proof}
The ``only if'' direction is straight-forward since the inequalities are sound w.r.t.\ the semantics of
statistical KBs. We focus on the ``if'' direction only.

Given a solution of the integer program, we construct an interpretation $\Imc=(\Delta,\cdot^\Imc)$ 
as follows.
We create a domain $\Delta$ with $x_\top$ elements, and partition it such that for every
type $\theta\in\Theta(\Kmc)$, there is a class $[[\theta]]$ containing exactly $x_\theta$ elements.
For every non-empty class, select a representative element $\delta_\theta\in[[\theta]]$.

The interpretation function $\cdot^\Imc$ maps every concept name $A$ to the set
\[
A^\Imc := \bigcup_{A\in\theta}[[\theta]].
\]
Given a non-empty class $[[\theta]]$ such that $A\in\theta$ and $A\sqsubseteq\exists r.B\in \Tmc$,
let $\tau$ be a type such that $B\in\tau$, $x_\tau>0$, and for every $\exists r.C\sqsubseteq D\in\Tmc$, if
$D\notin\theta$, then $C\notin\tau$. 
Notice that such a $\tau$ must exist because the solution must satisfy at least one restriction
in each $\Fmc_{A,B,C,D}$.
We define $r^\theta_{A,B}:= \theta\times\{\delta_\tau\}$ and set
\[
r^\Imc := \bigcup_{A\in\theta,A\sqsubseteq\exists r.B\in\Tmc}r^\theta_{A,B}.
\]
It remains to be shown that \Imc is a model of \Kmc.

Notice that for two concept names $A,B$, it holds that 
$(A\sqcap B)^\Imc=\bigcup_{A,B\in\theta}[[\theta]]$ and
hence $[A\sqcap B]^\Imc|=\sum_{A,B\in\theta}x_\theta$.
Given a conditional statement $(A\mid B)[\ell,u]\in\Cmc$, since the solution must satisfy the
inequality~\eqref{eqn:conditional}, it holds that 
\[
\ell\cdot [B]^\Imc \le [A\sqcap B]^\Imc \le u\cdot [B]^\Imc.
\]
For a GCI $A_1\sqcap A_2\sqsubseteq B\in \Tmc$, by the inequality~\eqref{eqn:conlast} it follows 
that for every type $\theta$ containing both $A_1,A_2$, but not $B$, $[[\theta]]=\emptyset$. 
Hence $A_1^\Imc\cap A_2^\Imc\subseteq B^\Imc$.
For every $A\sqsubseteq\exists r.B\in\Tmc$, and every $\gamma\in\Delta$, if $\gamma\in A^\Imc$
then by construction there is an element $\gamma'$ such that $(\gamma,\gamma')\in r^\Imc$.

Finally, if $(\gamma,\gamma')\in r^\Imc$, then by construction there exists a type $\theta$
and an axiom $A\sqsubseteq \exists r.B\in\Tmc$ such that $\gamma\in[[\theta]]$ and 
$\gamma'=\delta_\tau$. Then, for every GCI $\exists r.C\sqsubseteq D\in \Tmc$, 
$\gamma'\in C^\Imc$ implies $C\in\tau$ and hence $D\in\theta$ which means that $\gamma\in D^\Imc$.
\qed
\end{proof}
Notice that the construction produces exponentially many integer programs, each of which uses 
exponentially many variables, measured on the size of the KB. Since satisfiability of integer linear programs 
is decidable in non-deterministic polynomial time on the size of the program, we obtain a non-deterministic 
exponential time upper bound for deciding consistency of statistical \el KBs.
\begin{theorem}
Consistency of statistical \el KBs is in {\sc NExpTime}.
\end{theorem}